\documentclass[a4paper,11pt,oneside]{article}
\usepackage[margin=1in]{geometry}

\usepackage{amssymb, natbib, algorithm2e}
\usepackage{tikz}
\usepackage{mathtools}
\usepackage[colorinlistoftodos,prependcaption]{todonotes} %
\usepackage{hyperref}

\usetikzlibrary{automata,positioning}
\numberwithin{equation}{section}
\hypersetup{
	colorlinks=true,
  linkcolor=red,          %
  citecolor=blue,         %
  filecolor=magenta,      %
  urlcolor=cyan           %
}

\newcommand{\overbar}[1]{\mkern 1.5mu\overline{\mkern-1.5mu#1\mkern-1.5mu}\mkern 1.5mu}
\newcommand{\pred}[1]{\boldsymbol{1}\left\{#1\right\}}
\newcommand{\R}{\mathbb{R}}

\newcommand{\N}{\mathbb{N}}

\newcommand{\zero}{\boldsymbol{0}}
\newcommand{\trn}{^\intercal}
\newcommand{\inv}{^{-1}}

\DeclareMathOperator{\diag}{diag}

\newcommand{\expo}[1]{\exp\left( #1 \right)}

\newcommand{\hide}[1]{}
\newcommand{\abs}[1]{\left| #1 \right|}
\newcommand{\paren}[1]{\left( #1 \right)}
\newcommand{\nrm}[1]{\left\Vert #1 \right\Vert}
\newcommand{\vertiii}[1]{{\left\vert\kern-0.25ex\left\vert\kern-0.25ex\left\vert #1 
    \right\vert\kern-0.25ex\right\vert\kern-0.25ex\right\vert}}
\newcommand{\tv}[1]{\nrm{#1}_{\mathsf{TV}}}
\newcommand{\TV}[1]{\vertiii{#1}}

\newcommand{\dkaz}[1]{\mathrm{Dist}_{\sqrt{\phantom{x}}}\left( #1 \right)}

\newcommand{\M}{\mathcal{M}}
\newcommand{\mc}{\boldsymbol{M}}
\newcommand{\bpi}{\boldsymbol{\pi}}
\newcommand{\bmu}{\boldsymbol{\mu}}
\newcommand{\bsigma}{\boldsymbol{\sigma}}
\newcommand{\bnu}{\boldsymbol{\nu}}

\newcommand{\refmc}{\overbar\mc}

\newcommand{\sg}{\gamma}
\newcommand{\asg}{\gamma_\star}
\newcommand{\pssg}{\sg_{\mathsf{ps}}}

\newcommand{\tmix}{t_{\mathsf{mix}}}

\newcommand{\X}{\boldsymbol{X}}
\newcommand{\PR}[2][]{\mathbf{P}_{#1}\left( #2 \right)}
\newcommand{\PN}[1]{\mathbf{P}_0\left( #1 \right)}
\newcommand{\PA}[1]{\mathbf{P}_1\left( #1 \right)}
\newcommand{\E}[2][]{\mathbf{E}_{#1}\left[ #2 \right]}
\newcommand{\Var}[2][]{\mathbf{Var}_{#1} \left[ #2 \right]}
\newcommand{\dist}{\boldsymbol{D}}
\newcommand{\refdist}{\overbar{\dist}}

\newcommand{\T}{\mathcal{T}}

\newcommand{\eps}{\varepsilon}

\DeclareMathOperator{\Geometric}{Geometric}

\newcommand{\IID}{_{\textrm{{\tiny \textup{IID}}}}}
\newcommand{\learn}{_{\textrm{{\tiny \textup{LEARN}}}}}
\newcommand{\ub}{_{\textrm{{\tiny \textup{UB}}}}}
\newcommand{\lb}{_{\textrm{{\tiny \textup{LB}}}}}

\newcommand{\G}{\mathcal{G}}

\newcommand{\gn}{|}
\newcommand{\vn}{\boldsymbol{n}}
\newcommand{\thit}[1]{U_{#1}}

\newcommand{\tmultj}[2]{H_{#1}^{(#2)}}
\newcommand{\refeta}{\bar{\etab}}

\newcommand{\beq}{\begin{eqnarray*}}
\newcommand{\eeq}{\end{eqnarray*}}
\newcommand{\beqn}{\begin{eqnarray}}
\newcommand{\eeqn}{\end{eqnarray}}

\newcommand{\etab}{\boldsymbol{\eta}}

\newcommand{\x}{\boldsymbol{x}}

\newcommand{\U}{\boldsymbol{U}}

\newcommand{\hcover}{_{\textrm{{\tiny \textup{COVER}}/2}}}

\newcommand{\eqdef}{\doteq}

\newcommand{\pimin}{\pi_\star}

\newcommand{\bigO}{\mathcal{O}}
\newcommand{\IO}{_{\textrm{{\tiny \textup{IO}}}}}
\newcommand{\bpp}{_{\textrm{{\tiny \textup{BPP}}}}}
\newcommand{\fix}{_{\textrm{{\tiny \textup{FIX}}}}}
\newcommand{\Cmc}{C_{\textrm{{\tiny \textup{MC}}}}}

\newcommand{\hcliq}{_{\textrm{{\tiny \textup{CLIQ}}/2}}}
\newcommand{\thalf}{T\hcliq}
\newcommand{\btau}{\boldsymbol{\tau}}
\renewcommand{\H}{\mathcal{H}}

\newcommand{\set}[1]{\left\{ #1 \right\}}
\newtheorem{theorem}{Theorem}[section]
\newtheorem{lemma}{Lemma}[section]
\newtheorem{definition}{Definition}[section]

\newtheorem{remark}{Remark}[section]
\newenvironment{proof}{\paragraph{Proof:}}{\hfill$\square$}

\title{Minimax Testing of Identity \\ to a Reference Ergodic Markov Chain}
\author{Geoffrey Wolfer \\ \texttt{geoffrey@post.bgu.ac.il} \and Aryeh Kontorovich \\ \texttt{karyeh@cs.bgu.ac.il}}

\begin{document}

\maketitle

\begin{abstract}%
We exhibit an efficient procedure for testing,
based on a single long state sequence,
whether an unknown Markov chain is identical to or $\eps$-far from a given reference chain.
We obtain nearly matching (up to logarithmic factors) upper and lower
sample complexity bounds for our notion of distance, which is based on total variation.
Perhaps surprisingly,
we discover that the sample complexity depends
solely on the
properties of the known reference chain and does not
involve
the unknown chain at all,
which is
not even assumed to be ergodic.

\end{abstract}

\section{Introduction}
\label{section:introduction}
Distinguishing whether an unknown distribution $\dist$ is identical to a reference one
$\refdist$ or is $\eps$-far from it in total variation (TV)
is a
special case
of statistical
\emph{property testing}.
For the iid case, it
is known
that a sample of size
$\tilde \Theta(\sqrt d/\eps^2)$,
where $d$ is the support size,
is both sufficient and necessary \citep{batu2001testing, valiant2017automatic}.
This is in contradistinction to the corresponding learning problem,
with the considerably higher
sample complexity
of
$\tilde \Theta(d/ \eps^2)$
(see, e.g., \citet{MR1741038, DBLP:conf/innovations/Waggoner15, DBLP:journals/corr/KontorovichP16}).
The Markovian setting has so far received no attention
in the property testing framework,
with
the
notable
exception of the recent work of \citet{daskalakis2017testing} and \citet{cherapanamjeri2019testing}, 
which we discuss in greater detail in Section~\ref{section:related-work}.
\citeauthor{daskalakis2017testing} ``initiate[d] the study of Markov chain testing'',
but
imposed the stringent constraint of being symmetric on both the 
reference and unknown chains.
In this paper,
we
only require ergodicity of the reference chain,
and
make no assumptions
on the unknown one
---
from which the tester receives a single long trajectory of observations
---
other than it having $d$ states.

\paragraph{Our contribution.}
We prove
nearly
matching upper and lower bounds on the sample complexity of the
testing problem in terms of the accuracy $\eps$ and the number of states $d$,
as well as parameters derived from the
stationary distribution and mixing.
We discover that for testing, only the
reference chain
affects the sample complexity,
and no assumptions (including ergodicity) need be made on the unknown one.
In particular, we exhibit
an efficient testing procedure, which, given
a Markovian sequence of length
\begin{equation}
\begin{split}
\label{eq:minimax-sample-complexity}
m =  \tilde{\bigO}\left( \frac{1}{\pimin} \max \set{  \frac{\sqrt{d}}{\eps^2}  , \tmix } \right)
,
\end{split}
\end{equation}
correctly identifies the unknown chain
with high probability,
where $\pimin$ and $\tmix$ are, respectively,
the minimum stationary probability and mixing time of the \emph{known} reference
chain.
We also derive an instance-specific version of the previous bound:
the factor
$\frac{\sqrt{d}}{\pimin}$ in \eqref{eq:minimax-sample-complexity}
can be replaced with the potentially much smaller quantity
$\max_{i\in[d]} \set{ \bpi(i)^{-1}  \nrm{\refmc(i, \cdot)}_{2/3} }$,
defined in \eqref{eq:2/3-def}.
Additionally, we construct two separate worst-case lower bounds of
$\Omega \left( d \tmix \right)$ and $\Omega \left( \frac{\sqrt{d}}{\pimin \eps^2} \right)$, exhibiting a
regime for which our testing procedure is unimprovable.

\section{Related work}
\label{section:related-work}
We consider {\em distribution testing} in the {\em property testing}
(within the more classical statistical {\em hypothesis testing}\footnote{A recent result in this vein is \citet{barsotti2016hypothesis}; see also references therein. }) framework --- a research program initiated by \citet{DBLP:conf/focs/BatuFRSW00}.

The special case of iid uniformity testing was addressed (for various metrics) by \citet{goldreich2011testing, paninski2008coincidence}. 
Extensions to iid identity testing for arbitrary finite distributions were then obtained \citep{goldreich2016uniform,diakonikolas2016collision}, including the instance-optimal tester of \citet{valiant2017automatic}, who showed that $\sqrt{d}$ may be replaced with
$\nrm{\refdist}_{2/3}$, the $(2/3)$-pseudo-norm of the reference distribution.

To our knowledge, \citet{daskalakis2017testing} were the first to consider the testing problem for Markov chains (see references therein for previous works addressing goodness-of-fit testing under Markov dependence). Their model is based on the pseudo-distance $\dkaz{\cdot, \cdot}$
defined
by
\citet{kazakos1978bhattacharyya}
as
\begin{equation*}
\begin{split}
  \dkaz{\mc, \mc'}
  = 1 - \rho\left( \left[ \mc , \mc' \right]_{\sqrt{\phantom{x}}} \right),
\end{split}
\end{equation*}
where $\paren{\left[ \mc , \mc' \right]_{\sqrt{\phantom{x}}}}_{(i,j)} = \sqrt{\mc(i,j)\mc'(i,j)} $ is the term-wise geometric mean of the transition kernels and $\rho$ is the largest eigenvalue in magnitude. This pseudo-distance has the property of vanishing on pairs of chains sharing an identical connected component.
\citeauthor{daskalakis2017testing}'s
sample complexity upper bound of
$\tilde{\bigO} \left( \mathrm{HitT}_{\refmc} + {d}/{\eps} \right)$ required knowledge of the hitting time of the reference chain, while their lower bound $\Omega( {d}/{\eps})$ involves no quantities related to the mixing rate at all.
The authors conjectured
that for $\dkaz{\cdot}$,
the correct
sample complexity is $\Theta ({d}/{\eps})$ --- i.e., independent of the mixing properties of the chain.
This conjecture was recently partially proven by \citet{cherapanamjeri2019testing},
who gave an upper
bound of $\tilde \bigO (d/\eps^4)$,
without dependence on
the hitting time.

The present paper compares favorably with \citeauthor{daskalakis2017testing} in that the latter requires both the reference and the unknown chains to be symmetric (and, a fortiori, reversible) as well as ergodic. We only require ergodicity of the reference chain and assume nothing about the unknown one. 

Additionally,
we obtain nearly sharp sample complexity bounds
in terms of the reference chain's mixing properties.
Finally,
our metric $\TV{\cdot}$ dominates the pseudo-metric $\dkaz{\cdot}$,
and
hence their identity testing problem is reducible to ours
(see Lemma~\ref{lemma:comparison-daskalakis} in the Appendix),
although the reduction does not preserve the convergence rate.

We note that the corresponding PAC-type learning problem for Markov chains was only recently
introduced in \citet{hao2018learning, wolfer2018minimax}.
The present paper uses the same notion of distance as the latter work, in which
the minimax sample complexity for learning was shown to be
of the order of $m = \tilde \Theta \left( \frac{1}{\pimin} \max \set{ \frac{d}{\eps^2}, \tmix } \right)$.
Our present results
confirm the intuition that
identity testing
is, statistically speaking, a significantly less demanding task than learning:
the former exhibits
a quadratic reduction in the bound's dependence on $d$ over the latter.

\section{Definitions and notation}
\label{section:definition-notation}
We define $[d] \eqdef \set{1,\ldots,d}$,
denote
the simplex of all distributions over $[d]$
by $\Delta_d$, and the collection of all $d\times d$ row-stochastic matrices by $\M_d$.
For $\bmu\in\Delta_d$, we will write either $\bmu(i)$ or $\mu_i$, as dictated by convenience. All vectors are rows unless indicated otherwise. 
For $n \in \N$, and any $\dist \in \Delta_d$ we also consider its $n$-fold \emph{product} $\dist^{\otimes n}$, i.e. $(X_1, \dots, X_n) \sim \dist^{\otimes n}$ is a shorthand for $(X_i)_{i \in [n]}$ being all mutually independent, and such that $\forall i \in [n], X_i \sim \dist$. 
A Markov chain on $d$ states being entirely specified by an initial distribution $\bmu\in\Delta_d$ and a row-stochastic transition matrix $\mc\in\M_d$, we identify the chain with the pair $(\mc,\bmu)$. Namely, writing $X_1^m$ for $(X_1, \dots, X_m)$, by $X_1^m \sim(\mc,\bmu)$, we mean that
\begin{equation*}
\PR{X_1^m = x_1^m }=\bmu(x_1)\prod_{t=1}^{m-1}\mc(x_t,x_{t+1}).
\end{equation*}
We write $\PR[\mc,\mu]{\cdot}$ to denote probabilities over sequences induced by the Markov chain $(\mc,\bmu)$,
and omit the subscript when it is clear from context. 
Taking
the null hypothesis
to be that $\mc=\refmc$ (i.e., the chain being tested is identical to the reference one),
$\PN{\cdot}$ will denote probability
in the completeness case, and $\PA{\cdot}$ in the soundness case.
The Markov chain $(\mc,\bmu)$ is {\em stationary} if $\bmu=\bpi$ for $\bpi=\bpi\mc$, and {\em ergodic} if
$\mc^k
>
0$ (entry-wise positive) for some $k\ge1$. If $\mc$ is ergodic,
it has a unique stationary distribution $\bpi$ and moreover the \emph{minimum stationary probability} $\pimin>0$, where
\begin{equation}
\label{eq:pistar-def}
\pimin \eqdef \min_{i\in[d]}\bpi(i).
\end{equation}
Unless noted otherwise, $\bpi$ is assumed to be the stationary distribution of the Markov chain in context.
The {\em mixing time} of a chain
is defined
as the number of
steps
necessary
for
its
state distribution
to be
sufficiently close to the stationary one
(traditionally taken to be within $1/4$):
\begin{equation}
\label{definition:tmix}
\tmix \eqdef \inf \set{t \geq 1 : \sup_{\bmu \in \Delta_{d}} \tv{\bmu \mc^{t} - \bpi} \leq \frac{1}{4} }
.
\end{equation}
We use the standard $\ell_1$ norm $\nrm{z}_1=\sum_{i\in[d]}|z_i|$, which, in the context of distributions (and up to a convention-dependent factor of $2$) corresponds to the total variation norm.
For $\boldsymbol{A}\in\R^{d\times d}$, define
\begin{equation*}
\label{eq:normdef}
\TV{\boldsymbol{A}} \eqdef
\frac{1}{2} \max_{i\in[d]}\tv{\boldsymbol{A}(i,\cdot)}.
\end{equation*}
Finally, we use standard $\bigO(\cdot)$, $\Omega(\cdot)$ and $\Theta(\cdot)$ order-of-magnitude notation, as well as their tilde variants $\tilde \bigO(\cdot)$, $\tilde\Omega(\cdot)$, $\tilde\Theta(\cdot)$ where lower-order log factors
in any parameter
are suppressed.
\begin{definition}
  An $(\eps,\delta)$-{\em identity tester} $\T$ for Markov chains with sample complexity function $m_0(\cdot)$ is an algorithm that takes as input a reference Markov chain $(\refmc, \bar{\bmu})$ and $\X=(X_1,\ldots,X_m)$ drawn from some unknown Markov chain $(\mc, \bmu)$, and outputs $\T=\T(d, \eps, \delta, \refmc, \bar{\bmu}, \X) \in \set{0,1}$   such that for $m \ge m_0(d, \eps,\delta,\refmc,\bar{\bmu})$, both $\mc=\refmc \Rightarrow \T =0 $ and $\TV{\mc-\refmc}>\eps \Rightarrow \T =1$ hold with probability at least $1-\delta$.
  (The probability is over the draw of $\X$ and  any internal randomness of the tester.)
\end{definition}
\noindent

\section{Formal results}
\label{section:formal-results}
Since the focus of this paper is on statistical
rather than computational complexity, we defer the (straightforward) analysis of the runtimes of our tester to the Appendix, Section~\ref{section:computational-complexity}.

\begin{theorem}[Upper bound]
\label{theorem:test-ub}
There exists an $(\eps,\delta)$-identity tester $\T$ (provided at Algorithm~\ref{algorithm:test}), which, for all $0 < \eps < 2$, $0 < \delta < 1 $, satisfies the following. If $\T$ receives as input a $d$-state ``reference'' ergodic Markov chain
$(\refmc,\bar{\bmu})$, 
as well as a sequence $\X=(X_1,\ldots,X_m)$ of length at least $m \ub$, drawn according to an unknown chain $\mc$ (starting from an arbitrary state), then it outputs $\T=\T(d, \eps, \delta, \refmc, \bar{\bmu}, \X) \in \set{0,1}$ such that
\begin{equation*}
\begin{split}
\mc = \refmc &\implies \T = 0 \qquad \\
\qquad \TV{\mc-\refmc}>\eps &\implies \T =1
\end{split}
\end{equation*}
holds with probability at least $1-\delta$. The sample complexity is upper-bounded by
\begin{equation*}
\begin{split}
  m \ub
  & = \frac{\Cmc}{\pimin} \max \set{ \frac{\sqrt{d}}{\eps^2}   \ln{\left(\frac{d}{\delta \eps}\right)} ,
    \tmix \ln{\left( \frac{d }{\delta \pimin} \right)} } \\
  & = \tilde{\bigO} \left( \frac{1}{\pimin} \max \set{ \frac{\sqrt{d}}{\eps^2}, \tmix } \right),
\end{split}
\end{equation*}
where
$\Cmc$ is a universal constant, and $\tmix$ and $\pimin$ are respectively the \emph{mixing time} \eqref{definition:tmix} and the minimum stationary probability \eqref{eq:pistar-def} of $\refmc$.
\end{theorem}

\begin{remark}
  An important feature of Theorem~\ref{theorem:test-ub} is that the sample complexity only depends on the (efficiently computable, see Section~\ref{section:computational-complexity}) properties of the known reference chain. No assumptions, such as  symmetry
  (as in \citet{daskalakis2017testing,
    cherapanamjeri2019testing})
  or even ergodicity, are made on the unknown Markov chain, and none of its properties appear in the bound. 
\end{remark}

\begin{remark}
  Our results indicate that in the regime where the $\frac{\tmix}{\pimin}$ term is not dominant, the use of optimized
identity
  iid testers as subroutines confers an $\tilde \bigO(\sqrt d)$-fold improvement over the naive testing-by-learning strategy. 
\end{remark}

\begin{remark}
  We note but do not pursue the fact that the logarithmic dependencies on $\delta$ in our upper bound could
  be improved via a combination of the techniques of \citet{DBLP:journals/corr/abs-1708-02728} and the reduction to uniformity testing of \citet{goldreich2016uniform}.
\end{remark}

\begin{theorem}[Instance-specific upper bound]
\label{theorem:test-ub-instance-optimal}
There exists an $(\eps,\delta)$-identity tester $\T$, which, for all $0 < \eps < 2$, $0 < \delta < 1 $, satisfies the following. If $\T$ receives as input a $d$-state ``reference'' ergodic Markov chain
$(\refmc,\bar{\bmu})$, 
as well as a sequence $\X=(X_1,\ldots,X_m)$ of length at least $m \ub$, drawn according to an unknown chain $\mc$ (starting from an arbitrary state), then it outputs $\T=\T(d, \eps, \delta, \refmc, \bar{\bmu}, \X) \in \set{0,1}$ such that
\begin{equation*}
\begin{split}
\mc=\refmc &\implies \T = 0 \\
\TV{\mc-\refmc}>\eps &\implies \T =1
\end{split}
\end{equation*}
holds with probability at least $1-\delta$. The sample complexity is upper-bounded by
\begin{equation*}
\begin{split}
  m \ub
  & = \tilde{\bigO} \left( \max \set{ \frac{\TV{\refmc}_{\bpi, 2/3}}{\eps^2}, \frac{\tmix}{\pimin} } \right),
\end{split}
\end{equation*}
where $\tmix$ and $\pimin$ are as in Theorem~\ref{theorem:test-ub}, and
\beqn
\label{eq:2/3-def}
\TV{\refmc}_{\bpi, 2/3} \eqdef \max_{i\in[d]} \set{
  \frac{\paren{\sum_{j\in[d]} \refmc(i,j)^{2/3}}^{3/2}}{\bpi(i)}}
.
\eeqn
\end{theorem}

\begin{remark}
  Since we always have $\TV{\refmc}_{\bpi, 2/3}\le\sqrt{d}/\pimin$,
  the instance-specific bound is always at least as sharp as the worst-case one
  in Theorem~\ref{theorem:test-ub}.
  It may, however, be considerably sharper, as illustrated by
  a simple random walk on a $d$-vertex, $\Delta$-regular graph \citep[Section~1.4]{levin2009markov},
  for which the instance-specific bound is
  $\tilde \bigO \left( d \max \set{ \frac{\sqrt{\Delta}}{\eps^2}, \tmix } \right)$
  --- a savings of roughly $\sqrt d$.
\end{remark}

\begin{theorem}[Lower bounds]
\label{theorem:test-lb}
For every $0 < \eps < 1/8$, $\tmix \geq 50$, and $d = 6k$, $k \ge2 $, there exists a $d$-state Markov chain $\refmc$ with mixing time $\tmix$ and stationary distribution $\bpi$ such that every $(\eps,1/10)$-identity tester for reference chain $\refmc$ must require in the worst case a sequence $\X=(X_1,\ldots,X_m)$ drawn from the unknown chain $\mc$ of length at least
\beq 
m \lb = \tilde \Omega \left( \max\left\{ \frac{\sqrt{d}}{\eps^2 \pimin}, d \tmix \right\} \right),
\eeq
where $\tmix, \pimin$ are as in Theorem~\ref{theorem:test-ub}.
\end{theorem}

\begin{remark}
  As the proof shows,
  for any $0 < \pimin < 1/(2(d+1))$, a testing problem can be constructed that
  achieves the $\frac{\sqrt{d}}{\eps^2 \pimin}$ component of the lower bound.
  Moreover, for doubly-stochastic $\refmc$, we have $\pimin = 1/d$, which shows
  that the upper bound cannot be improved in all parameters simultaneously.
\end{remark}

\section{Overview of techniques}
For both upper and lower bounds, we
survey
existing techniques,
describe
their limitations
vis-\`a-vis
our problem,
and
highlight
the key technical challenges as well as our solutions for overcoming these.

\subsection{Upper bounds}

\paragraph{Na\"ive approach: testing-by-learning.}
We mention this approach mainly
to establish
a baseline comparison.
\citet{wolfer2018minimax}
showed
that in order to $(\eps, \delta)$-learn an unknown $d$-state ergodic Markov chain $(\mc, \bmu)$
under
the
$\TV{\cdot}$ distance,
a single trajectory of length
$m \learn(\mc) = \tilde{\bigO} \left( \frac{1}{\pimin} \max \set{ \frac{d}{\eps^2}, \tmix } \right)$ is sufficient.
It follows that one can test identity with sample complexity
\begin{equation*}
\begin{split}
m = \max \set{ m \learn(\mc), m \learn(\refmc) }
.
\end{split}
\end{equation*}
This na\"ive bound,
aside
from
being
much looser than bounds 
provided in the present paper,
has the additional drawback of 
depending on the unknown $\mc$ and, in particular, being completely uninformative
when the latter is not ergodic.

\paragraph{Reduction to iid testing.}
Our upper bound in Theorem~\ref{theorem:test-ub}
is achieved via the stratagem of invoking an existing iid distribution identity tester as a black box (this is also the general approach of \citet{daskalakis2017testing}). Intuitively, given the reference chain $\refmc$, we can compute its stationary distribution $\bpi$ and thus know roughly how many visits to expect in each state. Further, computing the mixing time $\tmix$ gives us confidence intervals about these expected visits (similar to \citet{wolfer2018minimax}, via the concentration bounds of \citet{paulin2015concentration}). Hence, if a chain fails to visit each state a ``reasonable'' number of times, our tester
in Algorithm~\ref{algorithm:test} rejects it. Otherwise, given that state $i$ has been visited
as expected,
we can apply an iid identity tester to its conditional distribution. The unknown Markov chain passes the identity test if every state's conditional distribution passes its corresponding iid test.

A central technical challenge in executing this stratagem is the fact that
conditioning on the number of visits introduces dependencies on the sample,
thereby breaking the Markov property.
To get around this difficulty, we use a similar scheme as in \citet{daskalakis2017testing}.

Finally,
invoking the tester of
\citet{valiant2017automatic}
as a black box,
it is straightforward to sharpen
the worst-case bound in Theorem~\ref{theorem:test-ub}
to the instance-specific one in Theorem~\ref{theorem:test-ub-instance-optimal}.

\subsection{Lower bounds}
A lower bound of $\Omega(d/\eps^2)$ is immediate via a reduction from the testing problem of
\citet{daskalakis2017testing}
to ours (see Remark~\ref{remark:reduction-lower-bound}).
Although our construction for obtaining the sharper lower bound of
$\Omega(\sqrt{d} / (  \pimin \eps^2))$
shares some conceptual features with the constructions in
\citet{hao2018learning, wolfer2018minimax},
a
considerably more delicate
analysis
is required here.
Indeed, the
technique of
tensorizing the KL divergence,
instrumental in the lower bound of \citeauthor{wolfer2018minimax},
would yield (at best) a sub-optimal estimate of
$\Omega(1/(\pimin \eps^2))$
in our setting.
Intuitively, bounding TV via KL divergence is too crude for our purposes.
Instead,
we
take the approach of reducing the problem, via a
covering argument, to
one of
iid testing,
and construct a family of Markov chains
whose structure allows us to
recover the Markov property
even after conditioning on the number of visits to a
certain ``special''
state.
The main contribution for this argument is the decoupling technique of Lemma~\ref{lemma:tv-lb-make-product}.
The second lower bound is based on the construction of \citeauthor{wolfer2018minimax},
for which the mixing time and accuracy of the test can both be controlled independently.
Curiously,
the aforementioned argument cannot be invoked verbatim
for our problem, and so we introduce here the twist of considering half-covers
of the chains (Lemma~\ref{lemma:half-cover}),
concluding the argument with a two-point technique.
This adaptation
shaves
a
logarithmic factor
off
the corresponding learning problem.

\section{Proofs}

\subsection{Proof of Theorem~\ref{theorem:test-ub}}
In order to prove Theorem~\ref{theorem:test-ub}, we design a testing procedure, describe its algorithm, and further proceed with its analysis.

\begin{algorithm}[ht]
 \SetKwData{Visits}{Visits}
 \SetKwData{Transitions}{Transitions}
 \SetKwData{Reference}{$\refmc$}
 \SetKwFunction{DistIIDIdentityTester}{DistIIDIdentityTester}
 \KwIn{$d, \eps, \delta, \Reference, \bpi, (X_1, \dots, X_m)
   $
 }
 \KwOut{\textsc{Accept} $= 0$
or
\textsc{Reject} $= 1$
 }
 $\Visits\leftarrow
   \zero
   \in\R^d$
	
 \For{$t \leftarrow 1$ \KwTo $m - 1$}{
   $\Visits(X_t)\leftarrow\Visits(X_t) + 1$
 }
 \For{$i \leftarrow 1$ \KwTo $d$}{
		\If{ $\abs{\Visits(i) - (m-1) \bpi(i)} > (m-1) \bpi(i) / 2$}{
			  \KwRet \textsc{Reject}
		}
 }
 \For{$i \leftarrow 1$ \KwTo $d$}{
   $\Transitions\leftarrow
   \zero
   \in\R^d$
	
   \For{$t \leftarrow 1$ \KwTo $m-1$}{
     $\Transitions(j)\leftarrow\Transitions(j)+\pred{X_t=i} \pred{X_{t+1}=j}$
     }
		\If{\DistIIDIdentityTester{$\eps , d, \Reference(i,\cdot), \Transitions$}   is \textsc{Reject}}{
			  \KwRet \textsc{Reject}
		}
 }
 \KwRet \textsc{Accept}
 \caption{The testing procedure $\T$}
\label{algorithm:test}
\end{algorithm}

\subsubsection{The testing procedure}

For an infinite trajectory $X_1, X_2, X_3, \dots$ drawn from $\mc$, 
and for any $i \in [d]$, we denote the random hitting times to state $i$,
$$\tau^{(i)}_1 \eqdef \inf \set{ t \geq 1 : X_t = i},$$
and for $s > 1$,
$$\tau^{(i)}_s \eqdef \inf \set{ t > \tau^{(i)}_{s-1} : X_t = i}.$$
Fixing $\vn \in \N^d$
and $i \in [d]$, let us define,
following~\citet{daskalakis2017testing},
the mapping
\begin{equation*}
\begin{split}
\Psi_{\vn}^{(i)}: [d]^\infty &\to [d]^{\vn(i)}\\
(X_1, X_2, X_3, \dots ) &\mapsto X_{\tau^{(i)}_{1} + 1}, X_{\tau^{(i)}_{2} + 1}, \dots, X_{\tau^{(i)}_{\vn(i)} + 1},
\end{split}
\end{equation*}
which outputs, for a trajectory
drawn from
$\mc$, the $\vn(i)$ first states that have been observed immediately after hitting $i$. 
It is a consequence of the Markov property that the coordinates of $\Psi_{\vn}^{(i)}(X_1^{\infty})$
be independent and identically 
distributed according to the conditional distribution defined by the $i$th state of $\mc$. Namely,
$$\left(X_{\tau^{(i)}_{s} + 1}\right)_{s \in [\vn(i)]} \sim \mc(i, \cdot)^{\otimes \vn(i)}.$$

\noindent Remark: For an infinite trajectory, this mapping is \emph{well-defined} almost surely, provided the chain is \emph{irreducible}, 
while for a finite draw of length $m$, $\tau^{(i)}_s$ can be infinite for some $s$, such that 
proper definition of $\Psi_{\vn}^{(i)}$ is a random event that depends on $\vn, m$ and the mixing properties of the chain.

For $X_1^m \sim (\mc, \bmu)$,
we define our identity tester
$\T(\X)$
in terms of $d$
sub-testers
$\T^{(i)}(\X)$,  $i \in [d]$,
whose definition
we defer until further in the analysis.
Intuitively, each $\T^{(i)}$ requires at least a ``reasonable'' number of visits to $i$, 
i.e. a lower bound on $\vn(i)$.

\begin{equation*}
\begin{split}
& \T(X_1^m) \colon [d]^m \to \set{0, 1} \\
X_1^m &\mapsto 1 - \pred{
\begin{matrix}
   \forall i \in [d], \Psi_{\vn}^{(i)}(X_1^m) \text{ is well-defined } \\
   \text{ and } \\
	 \forall i \in [d], \T^{(i)}(\Psi_{\vn}^{(i)}(X_1^m)) = 0 \\
  \end{matrix}
} \\
\end{split}
\end{equation*}

\subsubsection{Analysis of the tester}

\paragraph{Completeness.}

Consider the two following events
\begin{equation*}
\begin{split}
\mathcal{E}_\infty &\eqdef \set{\forall i \in [d], \T^{(i)}(\Psi_{\vn}^{(i)}(X_1^\infty)) = 0  }, \\
\mathcal{E}_{\Psi_{\vn}} &\eqdef \set{\forall i \in [d], \Psi_{\vn}^{(i)}(X_1^m) \text{ is well-defined}}. \\
\end{split}
\end{equation*}
The probability that the tester correctly outputs $0$ for a trajectory 
sampled from the reference chain $\refmc$ is
\begin{equation*}
\begin{split}
\PN{\T(X_1^m) = 0} &\stackrel{(i)}{=} \PN{\forall i \in [d],\T^{(i)}( \Psi_{\vn}^{(i)}(X_1^m)) = 0 \text{ and } \mathcal{E}_{\Psi_{\vn}} } \\
&\stackrel{(ii)}{=} \PN{\forall i \in [d],\T^{(i)}( \Psi_{\vn}^{(i)}(X_1^\infty)) = 0 \text{ and } \mathcal{E}_{\Psi_{\vn}} } \\
&\stackrel{(iii)}{=} \PN{\mathcal{E}_\infty \cap \mathcal{E}_{\Psi_{\vn}}} \\
&\stackrel{(iv)}{\geq} \PN{\mathcal{E}_\infty} - \PN{\lnot \mathcal{E}_{\Psi_{\vn}}},\\
\end{split}
\end{equation*}
where $(i)$ is by definition of $\T(X_1^m)$, 
$(ii)$ stems from the fact that in the event where $\Psi_{\vn}^{(i)}(X_1^m)$ is well-defined,
$$\Psi_{\vn}^{(i)}(X_1^m) = \Psi_{\vn}^{(i)}(X_1^\infty)$$
holds, while $(iii)$ is by definition of $\mathcal{E}_\infty$, and
$(iv)$ is by the following covering argument
\begin{equation*}
\begin{split}
\mathcal{E}_\infty &\subset \left( \mathcal{E}_\infty \cap \mathcal{E}_{\Psi_{\vn}} \right) \cup \lnot \mathcal{E}_{\Psi_{\vn}}. \\
\end{split}
\end{equation*}
Further setting $\vn = \frac{1}{2} (m-1) \bpi$, where $\bpi$ is the stationary distribution of the
reference chain, and from an application of the union bound,
\begin{equation*}
\begin{split}
\PN{\lnot \mathcal{E}_{\Psi_{\vn}}} &\leq \sum_{i \in [d]} \PN{\Psi_{\vn}^{(i)}(X_1^m) \text{ is not well-defined}} \\
&\leq \sum_{i \in [d]} \PN{ N_i < \vn(i) }, \\
&\leq \sum_{i \in [d]} \PN{ \abs{N_i - \E[\bpi]{N_i}} >  (m-1) \bpi(i) / 2}, \\
\end{split}
\end{equation*}
where $N_i \eqdef \sum_{t = 1}^{m-1} \pred{X_t = i}$ is the number of visits to state $i$
(not counting the final state at time $m$).
For $m \geq c \frac{\tmix}{\pimin} \ln \left( \frac{d}{\delta \pimin} \right), c \in \R_+$,
\begin{equation}
\label{eq:completeness-hits}
\begin{split}
\PN{\lnot \mathcal{E}_{\Psi_{\vn}}} \leq \frac{\delta}{3},
\end{split}
\end{equation}
using the Bernstein-type concentration inequalities of 
\citet{paulin2015concentration} as made explicit in \citet[Lemma~5]{wolfer2018minimax}.
Observe that no properties of the unknown chain were invoked in this deduction.

We are left with lower bounding $\PN{\mathcal{E}_\infty}$, the probability that all state-wise testers
correctly output $0$ in the idealized case where they have have access to enough samples.
We first recall
some standard results (see for example \citealp{DBLP:conf/innovations/Waggoner15}).
\begin{lemma}[iid $(\eps, 2/5)$-testing to identity]
\label{lemma:fixed-iid-tester}
Let $\eps > 0$ and $d \in \N$. There exists a universal constant $C_{\IID}$ and a tester $\T_{\fix}$, 
such that for any reference distribution $\refdist \in \Delta_d$, and any unknown distribution $\dist \in \Delta_d$, for a sample of
size $m_{\IID, 2/5} \geq C_{\IID} \frac{\sqrt{d}}{\eps^2}$
drawn iid
from $\dist$,
$\T_{\fix}$ can distinguish between the cases
$\dist \equiv \refdist$ and $\tv{\dist - \refdist} > \eps$
with probability $
3/5$.
\end{lemma}
\begin{lemma}[BPP amplification]
\label{lemma:fixed-bpp-tester}
Given any $(\eps, 2/5)$-identity tester $\T$
for the iid case
with sample complexity $m_{\IID, 2/5}$,
and any $0 < \delta < 2/5$, we can construct (via a majority vote)
an amplified tester $\T_{\bpp}$ such that for $m \geq 18 \ln \left( \frac{2}{\delta} \right) m_{\IID, 2/5}$, $\T_{\bpp}$ can distinguish the cases $\dist \equiv \refdist$ and $\tv{\dist - \refdist} > \eps$ with confidence $1 - \delta$.
\end{lemma}

Since for $i \in [d]$, $\Psi_{\vn}^{(i)}(X_1^\infty) \sim \mc(i, \cdot)^{\otimes \vn(i)}$,
and from a union bound,
we can invoke Lemma~\ref{lemma:fixed-iid-tester} 
and Lemma~\ref{lemma:fixed-bpp-tester}. 
Ensuring that 
$\forall i \in [d], \vn(i) \geq c \frac{\sqrt{d}}{\eps^2} \ln \left( \frac{1}{\delta} \right), c \in \R_+$,
promises that
\begin{equation}
\label{eq:completeness-error-iid}
\begin{split}
\PN{\mathcal{E}_\infty} \geq 1 - \delta/3.
\end{split}
\end{equation}

\paragraph{Soundness.}

In this case, recall that $X_1^m \sim \mc$ such that
$\exists i_0 \in [d]$
$$\tv{\refmc(i_0, \cdot) - \mc(i_0, \cdot)} > 2\eps.$$
Following similar arguments as in the completeness case, we upper bound the error probability of the tester
\begin{equation*}
\begin{split}
\PA{\T(X_1^m) = 0} &= \PA{\forall i \in [d],\T^{(i)}( \Psi_{\vn}^{(i)}(X_1^m)) = 0 \text{ and } \mathcal{E}_{\Psi_{\vn}}  } \\
&= \PA{\forall i \in [d],\T^{(i)}( \Psi_{\vn}^{(i)}(X_1^\infty)) = 0 \text{ and } \mathcal{E}_{\Psi_{\vn}}  } \\
&\leq \PA{\forall i \in [d],\T^{(i)}( \Psi_{\vn}^{(i)}(X_1^\infty)) = 0 } \\
&\leq \PA{\T^{(i_0)}( \Psi_{\vn}^{(i_0)}(X_1^\infty)) = 0 }. \\
\end{split}
\end{equation*}
From Lemma~\ref{lemma:fixed-iid-tester} 
and Lemma~\ref{lemma:fixed-bpp-tester}, for
$\vn(i_0) \geq c \frac{\sqrt{d}}{\eps^2} \ln \left( \frac{1}{\delta} \right)$,
it is the case
\begin{equation}
\label{eq:soundness-bound}
\begin{split}
\PA{\T^{(i_0)}( \Psi_{\vn}^{(i_0)}(X_1^\infty)) = 0 } \leq \delta/3.
\end{split}
\end{equation}
Finally, combining \eqref{eq:completeness-hits}, \eqref{eq:completeness-error-iid} and  \eqref{eq:soundness-bound} 
finishes proving the theorem.
\hfill$\square$

\subsection{Proof of Theorem~\ref{theorem:test-ub-instance-optimal}}
This claim follows immediately from the analysis of
the iid instance-optimal tester
\citep{valiant2017automatic},
which
is invoked to test
the conditional distributions
of each state.

\begin{lemma}[\citealt{valiant2017automatic}]
\label{lemma:fixed-iid-tester-instance-optimal}
Let $\eps > 0$ and $d \in \N$. There exists a universal constant $C_{\IO}$, such that for any reference distribution $\refdist \in \Delta_d$, there exists a tester $\T_{\fix, \IO}$ such that for any unknown distribution $\dist \in \Delta_d$, for a sample of
size $m_{\IO, 2/5} \geq C_{\IO} \frac{\nrm{\refdist}_{2/3}}{\eps^2}$
drawn iid from $\dist$, $\T_{\fix, \IO}$ can distinguish between the cases
$\dist \equiv \refdist$ and $\tv{\dist - \refdist} > \eps$
with probability $
3/5$.
\end{lemma}

We simply have to ensure that for each state $i$, $\vn(i) \geq C_{\IO} \frac{\nrm{\refmc{i, \cdot}}_{2/3}}{\eps^2} $,
i.e. $\frac{1}{2} (m-1) \bpi(i) \geq C_{\IO} \frac{\nrm{\refmc{i, \cdot}}_{2/3}}{\eps^2}$, whence the theorem.

\hfill$\square$

\subsection{Proof of Theorem~\ref{theorem:test-lb}, \texorpdfstring{lower bound in $\Omega \left( \frac{\sqrt{d}}{\eps^2 \pimin} \right)$}{proximity lower bound}.}
The metric domination result in
Lemma~\ref{lemma:comparison-daskalakis}
immediately implies
a lower bound
of $\Omega \left( {d}/{\eps^2} \right)$ (see Remark~\ref{remark:reduction-lower-bound}).
We now construct two independent and more delicate lower bounds
of $\Omega \left( \frac{\sqrt{d}}{\eps^2 \pimin} \right)$ and
$\Omega \left(d \tmix \right)$.
Let $\M_{d,\tmix,\pimin}$
be the collection of all $d$-state Markov chains whose stationary distribution is minorized by
$\pimin$ and whose mixing time is at most $\tmix$.
Our goal is to lower bound
the minimax risk,
defined by
\begin{equation}
\label{eq:minimax-risk}
\begin{split}
\mathcal{R}_m \eqdef \inf_{\T} \sup_{\refmc, \mc} \left[ \PN{\T = 1} + \PA{\T = 0}\right],
\end{split}
\end{equation}
where
the $\inf$ is
over all testing procedures
$\T: (X_1, \dots, X_m) \mapsto \set{0,1}$,
and the $\sup$ is
over all $\refmc, \mc \in
\M_{d,\tmix,\pimin}
$ such that
$ \TV{\refmc - \mc} > \eps$.

The analysis is simplified by considering
$(d+1)$-state Markov chains with $d$ even;
an obvious modification of the proof
handles the case of odd $d$.
Fix $0 < p_\star \leq 1/(2(d+1))$ and $0 < \eps < 1/2$, and define
$\boldsymbol{p} \in\Delta_{d+1}$ by 
\begin{equation}
\label{eq:definition-initial-distribution-for-G}
\boldsymbol{p}(d+1) = p_\star \text{ and } \boldsymbol{p}(i)=(1- p_\star)/d, \qquad i\in[d]. 
\end{equation}
Define the collection of $(d+1)$-state Markov chain transitions matrices,
\begin{equation*}
\label{eq:Gpi-def}
\begin{split}
\G_{p_\star} &\eqdef \set{\mc_{\etab} : \etab  \in \Delta_d}, \\
\mc_{\etab}
&\eqdef \begin{pmatrix}
 \boldsymbol{p}(1) & \hdots & \boldsymbol{p}(d) & p_\star \\
  \vdots & \vdots & \vdots & \vdots\\
  \boldsymbol{p}(1) & \hdots & \boldsymbol{p}(d) & p_\star \\
 \etab(1) & \hdots & \etab(d) & 0 \\
\end{pmatrix}. 
\end{split}
\end{equation*}
The stationary distribution $\bpi$ of a chain of this family is given by
\begin{equation*}
\bpi(i) = \frac{\boldsymbol{p}(i) + p_\star \etab(i)}{1 + p_\star}, i \neq d + 1, \qquad \bpi(d+1) = \frac{p_\star}{1 + p_\star},
\end{equation*}
and for $p_\star < \frac{d \etab(i)}{d+1}, \forall i \in [d]$,
we have
$\pimin = \bpi(d+1)$.
We define the $n$th hitting time for state $i$ as the random variable $\tmultj{i}{n}\eqdef  \inf\set{t\ge1: \sum_{s=1}^t \pred{X_s=i}=n}$; in words, this is the first time $t$ at which state $i$ has been visited $n$ times. Suppose that $\X=(X_1,\ldots, X_m)\sim(\mc,\boldsymbol{p})$ for some $\mc\in\G_{p_\star}$.
For any $n\in\N$, the $n$th hitting time $\tmultj{d+1}{n}$ to state $d+1$,
stochastically dominates\footnote{
  A random variable $A$ {\em stochastically dominates} $B$ if $\PR{A\ge x}\ge\PR{B\ge x}$
  for all $x\in\R$.
} the random variable $\sum_{s=1}^n R_s$,
where each $R_s$ is an independent copy distributed as $\Geometric(p_\star)$.
To see this, consider a similar chain where the value $0$ in the last row
is replaced with $p_\star$ with $\etab$ appropriately re-normalized;
clearly, the modification
can only make it easier to reach state $d+1$.
Continuing, we compute
$\E{\sum_{s=1}^n R_s} = {n}/{p_\star}$
and
$\Var{\sum_{s=1}^n R_s} = {n(1-p_\star)}/{p_\star^2} \leq {n}/{p_\star^2}$.
The Paley-Zygmund inequality implies that for $m < {n}/({2 p_\star})$,
\begin{equation}
\label{eq:paley-zigmund-lb-proximity}
\begin{split}
  \PR{\tmultj{d+1}{n} > m} &\geq \PR{\sum_{s=1}^n R_s > m} \\
	&\geq \PR{\sum_{s=1}^n R_s >
    \frac12 \E{\sum_{s=1}^n R_s}} \\ 
  &\geq \paren{1 + \cfrac{4\Var{\sum_{s=1}^n R_s}}{
      \E{\sum_{s=1}^n R_s}^2}}\inv \\
  &\geq 1 - \frac{1}{1 + n/4} \geq \frac{1}{5}
.
\end{split}
\end{equation}
Define the random variable $N_{d+1}=\sum_{t=1}^{m}\pred{X_t=d+1}$, i.e. the number of visits to state $d+1$,
and consider
a reference Markov chain $\refmc \eqdef \mc_{\refeta}\in\G_{p_\star}$,
where $\refeta= \U_d$ is the $[d]$-supported uniform distribution.
Restricting the problem to a subset of the family $\G_{p_\star}$
satisfying the
$\eps$-separation condition
only makes it easier for the tester,
as does taking any mixture $\mc_\Sigma$ of chains of this class
in lieu of the $\sup$ in \eqref{eq:minimax-risk}.
More specifically,
we choose
\begin{equation*}
\mc_{\Sigma} \eqdef \frac{1}{2^{d/2}} \sum_{\bsigma \in \set{-1, 1}^{d/2}} \mc_{\bsigma},
\end{equation*}
where $\mc_{\bsigma}(i, \cdot) = \boldsymbol{p}$ for $i \in [d]$, and 
\begin{equation*}
\begin{split}
\mc_{\bsigma}(d+1, \cdot) = \dist_{\bsigma} = \left( \frac{1 + \sigma_1 \eps}{d}, \frac{1 - \sigma_1 \eps}{d}, \dots, \frac{1 + \sigma_{d/2} \eps}{d}, \frac{1 - \sigma_{d/2} \eps}{d} \right).
\end{split}
\end{equation*}
By construction,
for all $\bsigma \in \set{-1, 1}^{d/2}$,
we have
$\TV{\mc_{\bsigma} - \refmc} = \nrm{\dist_{\bsigma} - \U_d}_1 = \eps$.
We start both chains with distribution $\boldsymbol{p}$,
defined in \eqref{eq:definition-initial-distribution-for-G}.
Now \eqref{eq:paley-zigmund-lb-proximity} implies that
for any $n \in \N$ and $m < n/({2 p_\star})$, any testing procedure
$\T: (X_1, \dots, X_m) \mapsto \set{0, 1}$ verifies
\begin{equation*}
\begin{split}
  \PR[\refmc]{\T = 1} + \PR[\mc_\Sigma]{\T = 0} \geq  \frac{1}{5} \bigg( \PR[\refmc]{\T = 1 \gn \mathcal{E}_n} + \PR[\mc_\Sigma]{\T = 0 \gn \mathcal{E}_n} \bigg),
\end{split}
\end{equation*}
where we wrote $\mathcal{E}_n \eqdef \set{ N_{d+1} \leq n }$.
It follows from
\citet[Chapter~16, Section~4]{le2012asymptotic} that
\begin{equation*}
\begin{split}
&\mathcal{R}_m \geq \frac{1}{5} \left( 1 - \tv{\PR[\mc_{\Sigma}]{ \X \gn \mathcal{E}_n} - \PR[\refmc]{ \X \gn \mathcal{E}_n}} \right),
\end{split}
\end{equation*}
and so it remains to upper bound a total variation distance.
For any $\mc,\mc'\in\G_{p_\star}$, the statistics
of the induced state sequence only differ in the visits to state $d+1$.

At this point,
we would like to
invoke an iid testing lower bound
--- but are cautioned against doing so naively, as conditioning on the number of visits to a state
breaks the Markov property.
Instead,
in Lemmas~\ref{lemma:factored-chain-visits}, \ref{lemma:conditional-upper-control-by-max} and \ref{lemma:tv-lb-make-product}
we develop a decoupling technique,
which yields
\begin{equation*}
\begin{split}
\tv{\PR[\mc_{\Sigma}]{ \X \gn \mathcal{E}_n} - \PR[\refmc]{ \X \gn \mathcal{E}_n}} \leq \tv{\dist_\Sigma^{\otimes n} - \U_d^{\otimes n}}.
\end{split}
\end{equation*}
We shall make use of
\citet[Theorem~4]{paninski2008coincidence}, which states:
\begin{equation*}
\tv{\dist_\Sigma^{\otimes n} - \U_d^{\otimes n}} \leq \sqrt{\expo{\frac{n^2 \eps^4}{d}}}.
\end{equation*}
It follows that
\begin{equation*}
\mathcal{R}_m \geq \frac{1}{5} \left( 1 - \frac{1}{2} \sqrt{\expo{\frac{n^2 \eps^4}{d}} - 1} \right). \\
\end{equation*}
Finally, for the
mixture of chains and parameter regime in question,
we have
$\bpi(d+1) = \pimin$ and $p_\star/2 \leq \pimin \leq p_\star$,
so that for $\delta < 1/10, m < \frac{n}{2 p_\star}$ and $n \leq \frac{\sqrt{d}}{\eps^2} \sqrt{\ln \left( \frac{4}{C^2} (1 - 5 \delta)^2 \right)}$,
it follows that $\mathcal{R}_m \geq \delta$.
This implies
a lower bound of
$m = \Omega \left( \frac{\sqrt{d}}{\eps^2 \pimin} \right)$
for the testing problem.

\hfill$\square$

\subsection{Proof of Theorem~\ref{theorem:test-lb}, \texorpdfstring{lower bound in $\Omega \left(d \tmix \right)$}{mixing lower bound}.}
Let us recall
the construction of \citet{wolfer2018minimax}. Taking $0 < \eps \le 1/8$ and $d=6k$, $k\ge2$ fixed,  $0 < \eta < 1/48$ and $\btau\in\set{0,1}^{d/3}$, we define the block matrix
\beq
\mc_{\eta,\btau}
  = 
\begin{pmatrix}
C_\eta & R_{\btau} \\
R_{\btau}\trn & L_{\btau}
\end{pmatrix}
,
\eeq
where
$C_\eta\in\R^{d/3\times d/3}$,
$  L_{\btau}  \in\R^{2d/3\times2d/3}$,
and
$R_{\btau} \in\R^{d/3\times 2d/3}$
are given by
$$
  L_{\btau} = \frac{1}{8} \diag\left( 7 - 4 \tau_1 \eps, 7 + 4 \tau_1 \eps, \dots, 7 - 4 \tau_{d/3} \eps, 7 + 4 \tau_{d/3} \eps  \right)
,$$
\beq
C_\eta
  =
\begin{pmatrix}
\frac{3}{4} - \eta & \frac{\eta}{d/3 - 1} & \hdots & \frac{\eta}{d/3 - 1} \\
\frac{\eta}{d/3 - 1} & \frac{3}{4} - \eta & \ddots & \vdots \\
\vdots & \ddots & \ddots & \frac{\eta}{d/3 - 1} \\
\frac{\eta}{d/3 - 1} & \hdots & \frac{\eta}{d/3 - 1} & \frac{3}{4} - \eta \\
\end{pmatrix}
,
\eeq
\beq
R_{\btau} = \frac{1}{8}
\begin{pmatrix}
1 + 4 \tau_1 \eps & 1 - 4 \tau_1 \eps  & 0 & \hdots & \hdots & \hdots & 0 \\
0 & 0 & 1 + 4 \tau_2 \eps & 1- 4 \tau_2 \eps  & 0 & \hdots & 0 \\
\vdots & \vdots & \vdots & \vdots & \vdots & \vdots & \vdots \\
0 & \hdots & \hdots & \hdots & 0 & 1 + 4 \tau_{d/3} \eps & 1- 4 \tau_{d/3} \eps \\
\end{pmatrix}
.
\eeq
Holding $\eta$ fixed, define the collection
\begin{equation}
\label{eq:Heta}
\H_\eta = \set{\mc_{\eta,\btau}: \btau \in \set{0,1}^{d/3} }
\end{equation}
of ergodic and symmetric stochastic matrices. Suppose that $\X=(X_1,\ldots, X_m) \sim(\mc,\bmu)$,
where $\mc \in\H_{\eta}$, and $\bmu$ is the uniform distribution over the inner clique
nodes,
indexed by $\set{1, \dots d/3}$.
Define
the random variable $T\hcliq$ to be the first time some half of the states in the inner clique were visited,
\beqn
\label{eq:Thcliq}
T\hcliq =
\inf\set{t\ge1:
  \abs{\set{X_1,\ldots,X_t}\cap[d/3]}
  =d/6}.
\eeqn
Lemma~\ref{lemma:half-cover} lower bounds the half cover time:
\begin{equation}
\label{eq:half-cover}
\begin{split}
m \leq \frac{d}{120 \eta} &\implies \PR{\thalf > m} \geq \frac{1}{5},
\end{split}
\end{equation}
while \citet[Lemma~6]{wolfer2018minimax} 
establishes the key property that any element $\mc$ of $\H_\eta$ satisfies
\beqn
\label{eq:control-spectral-gap}
\tmix(\mc)=\tilde\Theta(1/\eta).
\eeqn
Let us
fix some $i_\star \in [d]$, choose as reference $\refmc \eqdef \mc_{\eta,\zero}$ and as an alternative hypothesis $\mc \eqdef \mc_{\eta, \btau}$, with $\tau_i = \pred{i = i_\star}$. Take both chains to have the uniform distribution $\bmu$ over the clique nodes as their initial one. It is easily verified that $\TV{\refmc - \mc} = \eps$, so that
\begin{equation}
\mathcal{R}_m \geq \inf_{\T} \left[ \PN{\T = 1 | \thalf > m} \PN{\thalf > m} +  \PA{\T = 0 | \thalf > m}\PA{\thalf > m} \right].
\end{equation}
Further,
for $m < \frac{d}{120 \eta}$, we have
\begin{equation}
\mathcal{R}_m \geq \frac{1}{5} \inf_{\T} \left[ \PN{\T = 1 | \thalf > m} +  \PA{\T = 0 | \thalf > m} \right].
\end{equation}
Since
$\PR{X | Y} \geq \PR{X | Y, Z}\PR{Z | Y}$,
we have
\begin{equation}
\PN{\T = 1 | \thalf > m} \geq \PN{\T = 1 | \thalf > m, N_{i_\star} = 0} \PN{N_{i_\star} = 0 | \thalf > m}.
\end{equation}
Additionally, the symmetry of
our reference chain
implies that
$\PN{N_{i_\star} = 0 | \thalf > m} \geq 1/2$.
It follows,
via an analogous argument that $\PA{N_{i_\star} = 0 | \thalf > m} \geq 1/2$,
so that
\begin{equation}
\mathcal{R}_m \geq \frac{1}{10} \inf_{\T} \left[ \PN{\T = 1 | \thalf > m, N_{i_\star} = 0} +  \PA{\T = 0 | \thalf > m, N_{i_\star} = 0} \right].
\end{equation}
By Le Cam's theorem \citep[Chapter~16, Section~4]{le2012asymptotic},
\begin{equation}
  \label{eq:lecam}
\mathcal{R}_m \geq \frac{1}{10} \left[ 1 -  \tv{ \PN{\X | \thalf > m, N_{i_\star} = 0} - \PA{\X | \thalf > m, N_{i_\star} = 0} } \right].
\end{equation}
Other than 
state $i_\star$ and its connected outer nodes,
the reference chain
$\mc_{\eta,\zero}$ and the alternative chain $\mc_{\eta, \btau}$
are identical.
Conditional on $N_{i_\star} = 0$,
the outer states connected to $i_\star$ were never visited,
since
these are only connected to the rest of the chain via $i_\star$
and
our
choice of the initial distribution $\bmu$
constrains the
initial state
to the inner clique.
Thus,
the two distributions over sequences conditioned
on
$N_{i_\star} = 0$
are identical,
causing the
term
$\tv{ \PN{\cdot} - \PA{\cdot} }$
in
\eqref{eq:lecam} to vanish:
\begin{equation}
\mathcal{R}_m \geq \frac{1}{10},
\end{equation}
which proves a sample complexity lower bound of $\tilde \Omega(d \tmix)$.
Since our family of Markov chains has uniform stationary distribution
($\pimin=1/d$), this further proves
that the dependence on $\pimin$ in our bound
is in general not improvable.
\hfill$\square$

\section{Auxiliary lemmas}
The following standard combinatorial fact will be useful.
\begin{lemma}
\label{lemma:choose-non-consecutive}
Let $(m, n) \in \N^2$ such that $m + 1 \geq 2n$. Then there are ${ {m - n + 1} \choose n }$ ways of selecting $n$ non-consecutive integers from $[m]$.
\end{lemma}

\begin{lemma}
\label{lemma:factored-chain-visits}
For any $\mc \in \G_{p_\star}$ defined in \eqref{eq:Gpi-def}, started with initial distribution $\boldsymbol{p}$ defined in \eqref{eq:definition-initial-distribution-for-G},

$$ P(m, n, p_\star) \eqdef \PR[\mc, \boldsymbol{p}]{N_{d+1} = n} = \begin{cases} (1- p_\star)^n &\text{ if } n = 0 \\ p_\star^n(1-p_\star)^{m - 2n} \left[ {{m-n+1} \choose n} - {{m-n} \choose {n-1}} p_\star \right] &\text{ if } 1 \leq n \leq \frac{m+1}{2} \\ 0 &\text{ if } n > \frac{m+1}{2}\end{cases}.$$
\end{lemma}
\begin{proof}
  For any $\mc \in \G_{p_\star}$, we construct the following associated two-state Markov chain, with initial distribution $(1 - \boldsymbol{p}_\star, \boldsymbol{p}_\star)$,
  where all states $i \in [d]$ are merged into a single state, which we call $\overline{d+1}$, while state $d+1$ is
  kept
distinct.
Observe that this two-state Markov chain is the same for all $\mc \in \G_{p_\star}$, regardless of $\etab$,
and that the probability distribution of the number of visits to state $d+1$,
when sampling from $\mc$, is the same as when sampling from this newly constructed chain.

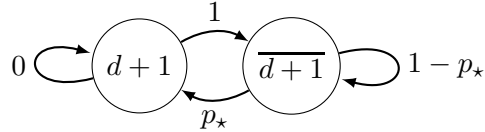
\begin{center}
\begin{tikzpicture}[node distance=2cm,->,>=latex,auto,
  every edge/.append style={thick}]
  \node[state] (1) {$d+1$};
  \node[state] (2) [right of=1] {$\overline{d+1}$};  
  \path (1) edge[loop left]  node{$0$} (1)
            edge[bend left]  node{$1$}   (2)
        (2) edge[loop right] node{$1-p_\star$}  (2)
            edge[bend left] node{$p_\star$}     (1);
\end{tikzpicture}
\end{center}

Let $m \geq 1$. The case where $n = 0$ is trivial as it corresponds to $n$ failures to reach the state $d+1$,
and there is only one such path.
When $n > \frac{m+1}{2}$, there is no path of length $m$ that contains $n$ visits to state $d+1$,
as any visit to this state almost surely cannot be directly followed by another visit to this same state.
It remains to analyze the final case where $1 \leq n \leq \frac{m+1}{2}$.
Take $(x_1, \dots, x_m)$
to be
a sample path in which the state $d+1$ was visited $n$ times. We consider two sub-cases.

\paragraph{The last state
  in the sample path is $d+1$:}

In the case where $x_m = d+1$, note that also necessarily $x_{m-1} = \overline{d+1}$. The $n-1$ previous visits to state $d+1$ were followed by a probability 1 transition to state $\overline{d+1}$,
and the remaining transitions have value $1-p_\star$, so that
$\PR{\X = \x} = p_\star^{n-1} 1^{n-1} (1-p_\star)^{m - (n-1) - (n-1) - 1} p_\star = p_\star^{n} (1 - p_\star)^{m - 2n + 1}$.
Since the last two
states
in the sequence
are fixed and known,
$x_m = d + 1$, $x_{m-1} = \overline{d+1}$,
counting the number of
such
paths
amounts to
counting the number of subsets of
$m - 2$ of size $n-1$ such that no two elements are consecutive, i.e.
${{(m-2) - (n-1) + 1}\choose {n-1}} = {{m-n} \choose {n-1}}$ (Lemma~\ref{lemma:choose-non-consecutive}).

\paragraph{The last state
  in the sample path is $\overline{d+1}$:}
By reasoning similar to above, such paths have probability
$p_\star^n (1 - p_\star)^{m - 2n}$.
To count such paths, consider all possible subsets of $m$ of size $n$ such that no two elements are consecutive, and subtract the count of paths in the other case where the last state was $d+1$. There are then ${{m - n + 1} \choose {n}} - {{m -n } \choose {n - 1}} = {{m-n} \choose {n}}$ such paths. \\\\ 
It follows that 
\begin{equation}
\begin{split}
P(m, n, p_\star) &= {{m-n} \choose {n-1}} p_\star^{n} (1 - p_\star)^{m - 2n + 1} + {{m-n} \choose {n}} p_\star^{n} (1 - p_\star)^{m - 2n} \\
& = p_\star^n(1-p_\star)^{m - 2n} \left[ {{m-n+1} \choose n} - {{m-n} \choose {n-1}} p_\star \right].
\end{split}
\end{equation}
\end{proof}

\begin{lemma}
\label{lemma:conditional-upper-control-by-max}
Let $\mc_1, \mc_2 \in \G_{p_\star}$,
defined in
\eqref{eq:Gpi-def}, and start both chains with initial distribution $\boldsymbol{p}$ defined in \eqref{eq:definition-initial-distribution-for-G}.
For arbitrary $(m, n) \in \N^2$
such that $m + 1 \geq 2n$,
let $N_{d+1} = \sum_{t=1}^{m} \pred{X_t = d+1}$
be the number of visits to state $(d+1)$.
Then, for trajectories $\X = (X_1, \dots, X_m)$ sampled from either chain, we have
\begin{equation}
\begin{split}
& \tv{\PR[\mc_1]{\X \mid N_{d+1} \leq n} - \PR[\mc_2]{\X \mid N_{d+1} \leq n}} \\
 &\leq \tv{\PR[\mc_1]{\X \mid N_{d+1} = n} - \PR[\mc_2]{\X \mid N_{d+1} = n }}.
\end{split}
\end{equation}
\end{lemma}
\begin{proof}
Partitioning over all possible number of visits to $d+1$ for $\mc_1$,
\begin{equation}
\begin{split}
\PR[\mc_1]{\X \mid N_{d+1} \leq n} &= \sum_{k = 0}^{\infty} \PR[\mc_1]{\X \mid N_{d+1} \leq n, N_{d+1} = k} \PR[\mc_1]{N_{d+1} = k \mid N_{d+1} \leq n} \\
&= \sum_{k = 0}^{n} \PR[\mc_1]{\X \mid N_{d+1} = k} \PR[\mc_1]{N_{d+1} = k \mid N_{d+1} \leq n} \\
&= \sum_{k = 0}^{n} \PR[\mc_1]{\X \mid N_{d+1} = k} \frac{\overbrace{\PR[\mc_1]{N_{d+1} \leq n \mid N_{d+1} = k}}^{=1} \PR[\mc_1]{N_{d+1} = k}}{\PR[\mc_1]{N_{d+1} \leq n}}
\end{split}
\end{equation}
From Lemma~\ref{lemma:factored-chain-visits}, we have
\begin{equation}
\begin{split}
\PR[\mc_1]{N_{d+1} = k} = \PR[\mc_2]{N_{d+1} = k} = P(m, k, p_\star) \\
\PR[\mc_1]{N_{d+1} \leq n} = \PR[\mc_2]{N_{d+1} \leq n} = \sum_{s=0}^{n} P(m, s, p_\star),
\end{split}
\end{equation}
and subsequently,
\begin{equation}
\begin{split}
& \tv{\PR[\mc_1]{\X \mid N_{d+1} \leq n} - \PR[\mc_2]{\X \mid N_{d+1} \leq n}} \\
& = \tv{ \sum_{k = 0}^{n} \bigg(  \PR[\mc_1]{\X \mid N_{d+1} = k} - \PR[\mc_2]{\X \mid N_{d+1} = k} \bigg) \frac{P(m, k, p_\star)}{\sum_{s=0}^{n} P(m, s, p_\star)} } \\
& \leq \sum_{k = 0}^{n} \tv{\PR[\mc_1]{\X \mid N_{d+1} = k} - \PR[\mc_2]{\X \mid N_{d+1} = k}} \frac{P(m, k, p_\star)}{\sum_{s=0}^{n} P(m, s, p_\star)} \\
& \leq \max_{k \in \set{0, \dots, n}} \tv{\PR[\mc_1]{\X \mid N_{d+1} = k} - \PR[\mc_2]{\X \mid N_{d+1} = k}} \\
&\leq \tv{\PR[\mc_1]{\X \mid N_{d+1} = n} - \PR[\mc_2]{\X \mid N_{d+1} = n }}.
\end{split}
\end{equation}
\end{proof}

The following
lemma shows that for the family $\G_{p_\star}$ of chains constructed in \eqref{eq:Gpi-def}, conditioned on the number of visits to state $d+1$,
it is possible to
control
the total variation between two trajectories drawn from two chains of the class in terms of the total
variation between product distributions.
	
\begin{lemma}
\label{lemma:tv-lb-make-product}
Let $\mc_{\etab_1}, \mc_{\etab_2} \in \G_{p_\star}$ defined in \eqref{eq:Gpi-def}, both started with initial distribution $\boldsymbol{p}$ defined in \eqref{eq:definition-initial-distribution-for-G}. Then, for $1 \leq n \leq \frac{m+1}{2}$,
\begin{equation}
\tv{\PR[\mc_{\etab_1}]{\X \gn N_{d+1} = n} - \PR[\mc_{\etab_2}]{\X \gn N_{d+1} = n} } \leq  \tv{\etab_1^{\otimes n} - \etab_2^{\otimes n}}.
\end{equation}
\end{lemma}
\begin{proof}
Total variation and $\ell_1$ norm are equal up to a conventional factor of 2,
\begin{equation}
\begin{split}
  & 2 \tv{\PR[\mc_{\etab_1}]{\X \gn N_{d+1} = n} - \PR[\mc_{\etab_2}]{\X \gn N_{d+1} = n} } \\
	&= \sum_{\x = (x_1, \dots, x_m) \in [d+1]^m} \abs{ \PR[\mc_{\etab_1}]{\X = \x \gn N_{d+1} = n} - \PR[\mc_{\etab_2}]{\X = \x \gn N_{d+1} = n}} \\
\end{split}
.
\end{equation}
Notice now that 
\begin{equation}
\begin{split}
\PR[\mc_{\etab_1}]{\X = \x \gn N_{d+1} = n} &= \frac{\PR[\mc_{\etab_1}]{N_{d+1} = n \gn \X = \x} \PR[\mc_{\etab_1}]{\X = \x}}{\PR[\mc_{\etab_1}]{N_{d+1} = n}} \\
&= \frac{\pred{n_{d+1} = n} \PR[\mc_{\etab_1}]{\X = \x}}{\PR[\mc_{\etab_1}]{N_{d+1} = n}},
\end{split}
\end{equation}
and similarly for $\mc_{\etab_2}$, so that
\begin{equation}
\begin{split}
  &2 \tv{\PR[\mc_{\etab_1}]{\X \gn N_{d+1} = n} - \PR[\mc_{\etab_2}]{\X \gn N_{d+1} = n} } \\
	&= \sum_{\x \in [d+1]^m} \abs{\frac{\pred{n_{d+1} = n} \PR[\mc_{\etab_1}]{\X = \x}}{\PR[\mc_{\etab_1}]{N_{d+1} = n}} - \frac{\pred{n_{d+1} = n} \PR[\mc_{\etab_2}]{\X = \x}}{\PR[\mc_{\etab_2}]{N_{d+1} = n}}}.
\end{split}
\end{equation}
Invoking Lemma~\ref{lemma:factored-chain-visits}, write
$$P(m,n,p_\star) = \PR[\mc_{\etab_1}]{N_{d+1} = n} = \PR[\mc_{\etab_2}]{N_{d+1} = n}.$$ 
For $1 \leq n \leq \frac{m+1}{2}$,
\begin{equation}
\begin{split}
  &2 \tv{\PR[\mc_{\etab_1}]{\X \gn N_{d+1} = n} - \PR[\mc_{\etab_2}]{\X \gn N_{d+1} = n} } \\
	&= \frac{1}{P(m,n,p_\star)} \sum_{\x \in [d+1]^m} \pred{n_{d+1} = n} \abs{\PR[\mc_{\etab_1}]{\X = \x} - \PR[\mc_{\etab_2}]{\X = \x}} \\
  &= \frac{1}{P(m,n,p_\star)} \Bigg( \sum_{\substack{\x \in [d+1]^m \\ n_{d+1} = n \\ x_m = d+1}} \abs{\PR[\mc_{\etab_1}]{\X = \x} - \PR[\mc_{\etab_2}]{\X = \x}} \\
	& + \sum_{\substack{\x \in [d+1]^m \\ n_{d+1} = n \\ x_m \neq d+1}} \abs{\PR[\mc_{\etab_1}]{\X = \x} - \PR[\mc_{\etab_2}]{\X = \x}} \Bigg). \\
\end{split}
\end{equation}
Recall that it is impossible to visit state $d+1$ twice in a row. Computing the first sum,
\begin{equation}
\begin{split}
&\sum_{\substack{\x \in [d+1]^m \\ n_{d+1} = n \\ x_m = d+1}} \abs{\PR[\mc_{\etab_1}]{\X = \x} - \PR[\mc_{\etab_2}]{\X = \x}} \\
&= \sum_{\substack{S = (s_1, \dots, s_n) \\ S \subset [m] \\ s_n = m \\ i \neq j \implies \abs{s_i - s_j} > 1}} \sum_{(x_{s_1}, \dots, x_{s_{n-1}}) \in [d+1]^{n-1}} d^{m - 2n + 1} p_\star^{n} \left( \frac{1-p_\star}{d} \right)^{m - 2n + 1} \abs{ \prod_{k = 1}^{n-1} \etab_1(x_{s_k}) -  \prod_{k = 1}^{n-1} \etab_2(x_{s_k})} \\
&= {{(m-2) - (n-1) + 1}\choose{n-1}} p_\star^{n} \left( 1-p_\star \right)^{m - 2n + 1} \sum_{(x_{s_1}, \dots, x_{s_{n-1}}) \in [d]^{n-1}}  \abs{ \prod_{k = 1}^{n-1} \etab_1(x_{s_k}) -  \prod_{k = 1}^{n-1} \etab_2(x_{s_k})} \\
&= {{m - n}\choose{n-1}} p_\star^{n} \left( 1-p_\star \right)^{m - 2n + 1} 2 \tv{\etab_1^{\otimes n-1} - \etab_2^{\otimes n-1}}, \\
\end{split}
\end{equation}
where the second inequality is from Lemma~\ref{lemma:choose-non-consecutive}. Similarly for the second sum,
\begin{equation}
\begin{split}
&\sum_{\substack{\x \in [d+1]^m \\ n_{d+1} = n \\ x_m \neq d+1}} \abs{\PR[\mc_{\etab_1}]{\X = \x} - \PR[\mc_{\etab_2}]{\X = \x}} \\
&= \sum_{\substack{S = (s_1, \dots, s_n) \\ S \subset [m] \\ s_n \neq m \\ i \neq j \implies \abs{s_i - s_j} > 1}} \sum_{(x_{s_1}, \dots, x_{s_{n}}) \in [d+1]^{n}} d^{m - 2n} p_\star^{n} \left( \frac{1-p_\star}{d} \right)^{m - 2n} \abs{ \prod_{k = 1}^n \etab_1(x_{s_k}) -  \prod_{k = 1}^n \etab_2(x_{s_k})} \\
&= \left( {{m - n + 1}\choose{n}} - {{m - n}\choose{n-1}} \right) p_\star^{n} \left( 1-p_\star \right)^{m - 2n} \sum_{(x_{s_1}, \dots, x_{s_{n}}) \in [d]^{n}}  \abs{ \prod_{k = 1}^n \etab_1(x_{s_k}) -  \prod_{k = 1}^n \etab_2(x_{s_k})} \\
&= {{m - n}\choose{n}} p_\star^{n} \left( 1-p_\star \right)^{m - 2n} 2 \tv{\etab_1^{\otimes n} - \etab_2^{\otimes n}}.
\end{split}
\end{equation}
Hence,
\begin{equation}
\begin{split}
  & 2 \tv{\PR[\mc_{\etab_1}]{\X \gn N_{d+1} = n} - \PR[\mc_{\etab_2}]{\X \gn N_{d+1} = n} } P(m, n, p_\star) \\
	&= 2 p_\star^{n} \left( 1-p_\star \right)^{m - 2n} \left[ {{m - n}\choose{n-1}}  \left( 1-p_\star \right) \tv{\etab_1^{\otimes n-1} - \etab_2^{\otimes n-1}} + {{m - n}\choose{n}}   \tv{\etab_1^{\otimes n} - \etab_2^{\otimes n}} \right]\\
	&\leq 2 p_\star^{n} \tv{\etab_1^{\otimes n} - \etab_2^{\otimes n}}  \left( 1-p_\star \right)^{m - 2n} \left[ {{m - n}\choose{n-1}}  \left( 1-p_\star \right) + {{m - n}\choose{n}}    \right]\\	
	&= 2 \tv{\etab_1^{\otimes n} - \etab_2^{\otimes n}} P(m, n, p_\star) \text{ (Lemma~\ref{lemma:choose-non-consecutive})}. \\	
\end{split}
\end{equation}	
\end{proof}

\begin{lemma}[Cover time]
  \label{lemma:half-cover}
  For
  $\mc \in \H_{\eta}
  $
  [defined in \eqref{eq:Heta}],
  the ``half cover time'' random variable $T\hcliq$   [defined in \eqref{eq:Thcliq}]
satisfies
\begin{equation}
\begin{split}
m \leq \frac{d}{120 \eta} &\implies \PR{\thalf > m} \geq \frac{1}{5}.
\end{split}
\end{equation}
\end{lemma}
\begin{proof}
  The proof
pursues a strategy similar to
\citet{wolfer2018minimax},
which is adapted to ``half'' rather than ``full'' coverings.
  Let $\mc\in\H_\eta$ and $\mc_I \in \M_{d/3}$ be
  such that $\mc_I$ consists only in the inner clique of $\mc$,
  and each outer rim state got absorbed into its unique inner clique neighbor:
  $$\mc_I = \begin{pmatrix}
1 - \eta & \frac{\eta}{d/3 - 1} & \hdots & \frac{\eta}{d/3 - 1} \\
\frac{\eta}{d/3 - 1} & 1 - \eta & \ddots & \vdots \\
\vdots & \ddots & \ddots & \frac{\eta}{d/3 - 1} \\
\frac{\eta}{d/3 - 1} & \hdots & \frac{\eta}{d/3 - 1} & 1 - \eta \\
  \end{pmatrix}.$$
  By construction, it is clear that $\thalf$ is
almost surely
greater than the half cover time of $\mc_I$.
The latter
corresponds to a generalized coupon half collection time
$U\hcover = 1 + \sum_{i=1}^{d/6 - 1}U_i$ where $U_i$
is
the time increment between the $i$th and the $(i+1)$th
unique visited state.
Formally,
if $\X$ is a random walk according to $\mc_I$ (started from any state),
then
$U_{1} = \min \{ t > 1 : X_t \neq X_1 \}$
and
for $i>1$,
\begin{equation}
  U_{i} = \min \{ t >
  1
  :
  X_t \notin \{ X_1, \dots, X_{U_{i-1}} \} \} - U_{i-1}.
\end{equation}
The random variables $\thit{1}, \thit{2}, \dots, \thit{d/6 - 1}$
are independent and
$\thit{i} \sim \Geometric \left( \eta - \cfrac{(i-1)\eta}{d/3} \right)$,
whence
\begin{equation}
\begin{split}
  \E{\thit{i}} = \frac{d/3}{\eta(d/3 - i + 1)},
\qquad
  \Var{\thit{i}} = \cfrac{1 - \left( \eta - \cfrac{(i-1)\eta}{d/3} \right)}{\left( \eta - \cfrac{(i-1)\eta}{d/3} \right)^2}
\end{split}
\end{equation}
and
\begin{equation}
\begin{split}
  \E{U\hcover} \geq 1 + \cfrac{d/3}{\eta} (\sigma_{d/3} - \sigma_{d/6}),
\qquad
  \Var{U\hcover} \leq \cfrac{(d/3)^2}{\eta^2} \cfrac{\pi^2}{6},
\end{split}
\end{equation}
where $\sigma_{d} = \sum_{i = 1}^{d} \frac{1}{i}$, and $\pi = 3.1416\dots$.
Since
$\ln{(d+1)} \leq \sigma_d \leq 1 + \ln{d}$,
and for $d = 6k, k \geq 2$,
we have
$\sigma_{d} - \sigma_{d/2} \geq \ln{2}$
it follows that
\begin{equation}
\begin{split}
  \E{U\hcover} \geq \cfrac{d}{\eta} \cfrac{\ln{2}}{3} ,
\qquad
  \Var{U\hcover} \leq \cfrac{d^2}{\eta^2} \cfrac{\pi^2}{54}.
\end{split}
\end{equation}
Invoking
the Paley-Zygmund inequality with $\theta = 1 - \frac{\sqrt{15}}{6 \ln{2}}$,
yields
\begin{equation}
\begin{split}
  \PR{U\hcover > \theta\E{U\hcover}} \geq
  \paren{1 + \cfrac{\Var{U\hcover}}{(1 - \theta)^2 (\E{U\hcover})^2}}\inv \geq \frac{1}{5},
\end{split}
\end{equation}
so that for $m \leq \frac{d}{120 \eta}$
we have
$ \PR{\thalf > m } \geq \frac{1}{5}$.
\end{proof}

\section{Computational complexity}
\label{section:computational-complexity}
\paragraph{Time complexity of designing a test.}

In order to determine the size of the sample required in order to reach desired proximity $\eps$ and confidence $1-\delta$, and to run the test algorithm, one must first compute $\bpi$ and $\tmix$. These two quantities are relative to the reference Markov chain, for which we assume to have a full description of the transition matrix $\refmc$. Interestingly they can be computed \emph{offline}, only once per test definition, no matter how many observed trajectories we will want to test. 

The eigenproblem
$\bpi \refmc = \bpi$
can be solved in time
$\tilde \bigO (d^\omega)$, where $2 \leq \omega \leq 2.3728639$ \citep{le2014powers};
this is a method for recovering $\bpi$.

In the reversible case, it is well known \citep{levin2009markov} that the mixing time of an ergodic Markov chain is controlled by its \emph{absolute spectral gap} $\asg \eqdef 1 - \max \set{\lambda_2, \abs{\lambda_d}}$, where $\lambda_1 \geq \lambda_2 \geq \dots \geq \lambda_d$ is the ordered spectrum of $\refmc$, and minimum stationary probability $\pimin$:
\begin{equation}
\label{eq:control-tmix-reversible}
\left( \frac{1}{\asg} - 1 \right) \ln{2} \leq \tmix \leq\frac{\ln \left( 4/\pimin \right)}{\asg}.
\end{equation}
The full eigen-decomposition used to obtain $\bpi$ of cost $\tilde \bigO (d^\omega)$ also yields $\lambda_\star$.

In the non-reversible case, the relationship between the spectrum and the mixing time is not nearly as straightforward, and it is the \emph{pseudo-spectral gap} \citep{paulin2015concentration}, 
\begin{equation}
\pssg \eqdef \max_{k \in \N} \set{ \frac{\sg \left( (\refmc^\dagger)^k \refmc^k \right)}{k}},
\end{equation}
where $\refmc^\dagger$ is the time reversal \citep{fill1991eigenvalue} of $\refmc$, that gives effective control:
\begin{equation}
\label{eq:control-tmix-non-reversible}
  \frac{1}{2\pssg}\leq \tmix \leq \frac{1}{\pssg} \left(
  \ln \frac{1}{\pimin} 
  + 2 \ln 2 +1 \right).
\end{equation}
For any $k \in \N$, we have
$\frac{\sg \left( (\refmc^\dagger)^k \refmc^k \right)}{k} \leq 1/k$,
and so 
$\pssg$
is computable to within an additive error of $\eta$ in time $ \bigO ( d^3 /\eta)$.

The instance-specific test requires computing the $d$ values 
$\nrm{\refmc(i, \cdot)}_{2/3}$,
which is feasible in time $\bigO(d^2)$.

\paragraph{Time and space complexity of performing a test.}
The first step is to verify whether each state has been visited a sufficient amount of times $n_i, i \in [d]$,
which can be achieved in time $\bigO (m + d)$ and space $\bigO (d)$.
The second step
is to proceed with the state-wise testing strategy, whose runtime will
depend on the black-box iid tester.
For instance, the test of \citet{valiant2017automatic} can be performed in $\bigO (m + d)$ operations.
Using this as the iid tester, we get an overall time complexity of
$\bigO(d(d+m))$.

\section{Comparison with existing work}
\label{section:appendix-comparison}
Let us compare our distance $\TV{\cdot}$ to the one of \citet{daskalakis2017testing}:
\begin{equation}
\begin{split}
\dkaz{\mc, \mc'} = 1 - \rho\left( \left[ \mc , \mc' \right]_{\sqrt{\phantom{x}}} \right),
\end{split}
\end{equation}
where $\paren{\left[ \mc , \mc' \right]_{\sqrt{\phantom{x}}}}(i,j) = \left[\sqrt{\mc(i,j)\mc'(i,j)} \right] $ 
and $\rho$ is the spectral radius.

\begin{lemma}
  \label{lemma:comparison-daskalakis}
  For all
  $\mc,\mc'\in\M_d$,
  $\TV{\mc - \mc'} \geq 2 \dkaz{\mc, \mc'}$.
  Conversely, for $d\ge4$,
  there exist
  $\mc,\mc'\in\M_d$
  such that $\dkaz{\mc, \mc'} = 0$
  while
  $\TV{\mc - \mc'} =1$.

\end{lemma}
\begin{proof}
For $\bmu,\bmu'\in\Delta_d$, define the Hellinger distance
\begin{equation}
\begin{split}
H^2(\bmu, \bnu) = \frac{1}{2} \sum_{i=1}^{d} \left( \sqrt{\bmu(i)} -\sqrt{\bnu(i)} \right).
\end{split}
\end{equation}
For arbitrary $\mc,\mc'\in\M_d$, a standard calculation yields
\begin{equation}
\begin{split}
\TV{\mc - \mc'} &= \max_{i \in [d]} \tv{\mc(i,\cdot) - \mc'(i,\cdot)} \geq 2 \max_{i \in [d]} H^2\left(\mc(i,\cdot), \mc'(i,\cdot)\right) \\
&= 2 \max_{i \in [d]} \left(1 - \sum_{j=1}^{d}\sqrt{\mc(i,j)\mc'(i,j)}\right) = 2 \left(1 - \min_{i \in [d]} \sum_{j=1}^{d}\sqrt{\mc(i,j)\mc'(i,j)}\right) \\
&\geq 2 \left(1 - \rho\left( \left[ \mc , \mc' \right]_{\sqrt{\phantom{x}}} \right)\right) = 2 \dkaz{\mc, \mc'},
\end{split}
\end{equation}
where the second inequality follows from the Perron-Frobenius theorem \citep[Chapter 8]{MR1777382}. For the second claim, choose any $\theta \in [-1,1]$ and put
\begin{equation*}
\begin{split}
\mc = \frac{1}{2}
  \left( {\begin{array}{cccc}
   1 & 1 & 0 & 0 \\
   1 & 1 & 0 & 0 \\
   0 & 0 & 1 & 1 \\
   0 & 0 & 1+\theta & 1-\theta \\
  \end{array} } \right)
  ,
  \qquad
\mc'= \frac{1}{2}
  \left( {\begin{array}{cccc}
   1 & 1 & 0 & 0 \\
   1 & 1 & 0 & 0 \\
   0 & 0 & 1 & 1 \\
   0 & 0 & 1 & 1 \\
  \end{array} } \right).
\end{split}
\end{equation*}
Since they have an identical connected component, $\dkaz{\mc, \mc'} = 0$, whereas $\TV{\mc - \mc'} = |\theta|$.
\end{proof}

\begin{remark}
\label{remark:reduction-lower-bound}
It follows that $\eps$-identity testing with respect $\dkaz{\cdot, \cdot}$ reduces to $2 \eps$-identity testing with respect to $\TV{\cdot}$. In particular, the results of \citet{daskalakis2017testing} immediately imply a lower bound of $\Omega \left( {d}/{\eps^2} \right)$ for our testing problem.
\end{remark}

\section*{Acknowledgments}
We are thankful to John Lafferty for bringing this problem to our 
attention and to Nick Gravin for the insightful conversations.

\bibliography{bibliography}

\begin{thebibliography}{23}
\providecommand{\natexlab}[1]{#1}
\providecommand{\url}[1]{\texttt{#1}}
\expandafter\ifx\csname urlstyle\endcsname\relax
  \providecommand{\doi}[1]{doi: #1}\else
  \providecommand{\doi}{doi: \begingroup \urlstyle{rm}\Url}\fi

\bibitem[Anthony and Bartlett(1999)]{MR1741038}
M.~Anthony and P.~L. Bartlett.
\newblock \emph{{N}eural {N}etwork {L}earning: {T}heoretical {F}oundations}.
\newblock Cambridge University Press, Cambridge, 1999.
\newblock ISBN 0-521-57353-X.
\newblock \doi{10.1017/CBO9780511624216}.
\newblock URL \url{http://dx.doi.org/10.1017/CBO9780511624216}.

\bibitem[Barsotti et~al.(2016)Barsotti, Philippe, and
  Rochet]{barsotti2016hypothesis}
F.~Barsotti, A.~Philippe, and P.~Rochet.
\newblock Hypothesis testing for {m}arkovian models with random time
  observations.
\newblock \emph{Journal of Statistical Planning and Inference}, 173:\penalty0
  87--98, 2016.

\bibitem[Batu et~al.(2000)Batu, Fortnow, Rubinfeld, Smith, and
  White]{DBLP:conf/focs/BatuFRSW00}
T.~Batu, L.~Fortnow, R.~Rubinfeld, W.~D. Smith, and P.~White.
\newblock Testing that distributions are close.
\newblock In \emph{41st Annual Symposium on Foundations of Computer Science,
  {FOCS} 2000, 12-14 November 2000, Redondo Beach, California, {USA}}, pages
  259--269, 2000.
\newblock \doi{10.1109/SFCS.2000.892113}.
\newblock URL \url{https://doi.org/10.1109/SFCS.2000.892113}.

\bibitem[Batu et~al.(2001)Batu, Fischer, Fortnow, Kumar, Rubinfeld, and
  White]{batu2001testing}
T.~Batu, E.~Fischer, L.~Fortnow, R.~Kumar, R.~Rubinfeld, and P.~White.
\newblock Testing random variables for independence and identity.
\newblock In \emph{Foundations of Computer Science, 2001. Proceedings. 42nd
  IEEE Symposium on}, pages 442--451. IEEE, 2001.

\bibitem[Cherapanamjeri and Bartlett(2019)]{cherapanamjeri2019testing}
Y.~Cherapanamjeri and P.~L. Bartlett.
\newblock Testing markov chains without hitting.
\newblock \emph{Conference on Learning Theory, {COLT} 2019}, 2019.

\bibitem[Daskalakis et~al.(2018)Daskalakis, Dikkala, and
  Gravin]{daskalakis2017testing}
C.~Daskalakis, N.~Dikkala, and N.~Gravin.
\newblock Testing symmetric markov chains from a single trajectory.
\newblock In \emph{Proceedings of the 31st Conference On Learning Theory},
  volume~75 of \emph{Proceedings of Machine Learning Research}, pages 385--409.
  PMLR, 06--09 Jul 2018.

\bibitem[Diakonikolas et~al.(2018)Diakonikolas, Gouleakis, Peebles, and
  Price]{DBLP:journals/corr/abs-1708-02728}
I.~Diakonikolas, T.~Gouleakis, J.~Peebles, and E.~Price.
\newblock Sample-optimal identity testing with high probability.
\newblock In \emph{45th International Colloquium on Automata, Languages, and
  Automata}, 2018.

\bibitem[Diakonikolas et~al.(2019+)Diakonikolas, Gouleakis, Peebles, and
  Price]{diakonikolas2016collision}
I.~Diakonikolas, T.~Gouleakis, J.~Peebles, and E.~Price.
\newblock Collision-based testers are optimal for uniformity and closeness.
\newblock \emph{Chicago Journal of Theoretical Computer Science, to appear},
  2019+.

\bibitem[Fill(1991)]{fill1991eigenvalue}
J.~A. Fill.
\newblock Eigenvalue bounds on convergence to stationarity for nonreversible
  markov chains, with an application to the exclusion process.
\newblock \emph{The annals of applied probability}, pages 62--87, 1991.

\bibitem[Goldreich(2016)]{goldreich2016uniform}
O.~Goldreich.
\newblock The uniform distribution is complete with respect to testing identity
  to a fixed distribution.
\newblock In \emph{Electronic Colloquium on Computational Complexity (ECCC)},
  volume~23, page~1, 2016.

\bibitem[Goldreich and Ron(2011)]{goldreich2011testing}
O.~Goldreich and D.~Ron.
\newblock On testing expansion in bounded-degree graphs.
\newblock In \emph{Studies in Complexity and Cryptography. Miscellanea on the
  Interplay between Randomness and Computation}, pages 68--75. Springer, 2011.

\bibitem[Hao et~al.(2018)Hao, Orlitsky, and Pichapati]{hao2018learning}
Y.~Hao, A.~Orlitsky, and V.~Pichapati.
\newblock On learning markov chains.
\newblock In \emph{Advances in Neural Information Processing Systems}, pages
  648--657, 2018.

\bibitem[Kazakos(1978)]{kazakos1978bhattacharyya}
D.~Kazakos.
\newblock The {B}hattacharyya distance and detection between markov chains.
\newblock \emph{IEEE Transactions on Information Theory}, 24\penalty0
  (6):\penalty0 747--754, 1978.

\bibitem[Kontorovich and Pinelis(2019)]{DBLP:journals/corr/KontorovichP16}
A.~Kontorovich and I.~Pinelis.
\newblock Exact lower bounds for the agnostic probably-approximately-correct
  (pac) machine learning model.
\newblock \emph{Ann. Statist.}, 47\penalty0 (5):\penalty0 2822--2854, 2019.
\newblock ISSN 0090-5364.
\newblock \doi{10.1214/18-AOS1766}.

\bibitem[Le~Cam(2012)]{le2012asymptotic}
L.~Le~Cam.
\newblock \emph{Asymptotic methods in statistical decision theory}.
\newblock Springer Science \& Business Media, 2012.

\bibitem[Le~Gall(2014)]{le2014powers}
F.~Le~Gall.
\newblock Powers of tensors and fast matrix multiplication.
\newblock In \emph{Proceedings of the 39th international symposium on symbolic
  and algebraic computation}, pages 296--303. ACM, 2014.

\bibitem[Levin et~al.(2009)Levin, Peres, and Wilmer]{levin2009markov}
D.~A. Levin, Y.~Peres, and E.~L. Wilmer.
\newblock \emph{Markov chains and mixing times, second edition}.
\newblock American Mathematical Soc., 2009.

\bibitem[Meyer(2000)]{MR1777382}
C.~Meyer.
\newblock \emph{Matrix analysis and applied linear algebra}.
\newblock Society for Industrial and Applied Mathematics (SIAM), Philadelphia,
  PA, 2000.
\newblock ISBN 0-89871-454-0.
\newblock URL \url{https://doi.org/10.1137/1.9780898719512}.
\newblock With 1 CD-ROM (Windows, Macintosh and UNIX) and a solutions manual
  (iv+171 pp.).

\bibitem[Paninski(2008)]{paninski2008coincidence}
L.~Paninski.
\newblock A coincidence-based test for uniformity given very sparsely sampled
  discrete data.
\newblock \emph{IEEE Transactions on Information Theory}, 54\penalty0
  (10):\penalty0 4750--4755, 2008.

\bibitem[Paulin(2015)]{paulin2015concentration}
D.~Paulin.
\newblock Concentration inequalities for {M}arkov chains by {M}arton couplings
  and spectral methods.
\newblock \emph{Electronic Journal of Probability}, 20, 2015.

\bibitem[Valiant and Valiant(2017)]{valiant2017automatic}
G.~Valiant and P.~Valiant.
\newblock An automatic inequality prover and instance optimal identity testing.
\newblock \emph{SIAM Journal on Computing}, 46\penalty0 (1):\penalty0 429--455,
  2017.

\bibitem[Waggoner(2015)]{DBLP:conf/innovations/Waggoner15}
B.~Waggoner.
\newblock \emph{L\({}_{\mbox{p}}\)} testing and learning of discrete
  distributions.
\newblock In \emph{Proceedings of the 2015 Conference on Innovations in
  Theoretical Computer Science, {ITCS} 2015, Rehovot, Israel, January 11-13,
  2015}, pages 347--356, 2015.
\newblock \doi{10.1145/2688073.2688095}.
\newblock URL \url{http://doi.acm.org/10.1145/2688073.2688095}.

\bibitem[Wolfer and Kontorovich(2019)]{wolfer2018minimax}
G.~Wolfer and A.~Kontorovich.
\newblock Minimax learning of ergodic markov chains.
\newblock In \emph{Proceedings of the 30th International Conference on
  Algorithmic Learning Theory}, volume~98 of \emph{Proceedings of Machine
  Learning Research}, pages 904--930, Chicago, Illinois, 22--24 Mar 2019. PMLR.
\newblock URL \url{http://proceedings.mlr.press/v98/wolfer19a.html}.

\end{thebibliography}
\bibliographystyle{abbrvnat}

\end{document}